%% file: root.tex
\documentclass[letterpaper, 10 pt, conference]{ieeeconf}  

\IEEEoverridecommandlockouts                              

\usepackage{amsthm}

\usepackage{times}
\usepackage{cite}
\usepackage{multicol}
\usepackage[bookmarks=true]{hyperref}
\usepackage{bm}
\usepackage{amsmath,amsfonts}
\usepackage{algorithm}
\usepackage[noend]{algpseudocode} 
\usepackage{xcolor}
\usepackage{xspace}
\usepackage{graphicx}
\usepackage{makecell}
\usepackage{float}
\usepackage{stfloats}
\usepackage[caption=false,font=footnotesize,labelfont=rm,textfont=rm]{subfig}
\usepackage{multirow}
\newtheorem{theorem}{Theorem}
\newtheorem{definition}{Definition}
\newtheorem{remark}{Remark}

\newtheorem{example}{Example}

\usepackage{adjustbox}
\usepackage{comment}
\newcommand{\citep}{\cite}

\newcommand\ecost[1]{\ensuremath{}}
\usepackage{amsthm}
\begin{document}


\title{\LARGE \bf Conflict-Based Search for Multi-Agent Path Finding with Elevators}





\markboth{}{}
\author{Haitong He$^{1,2}$, Xuemian Wu$^1$, Shizhe Zhao$^1$, Zhongqiang Ren$^{1\dagger}$%
\thanks{$^1$ Global College, Shanghai Jiao Tong University, China. This work was supported by the Natural Science Foundation of Shanghai under Grant 24ZR1435900.}
\thanks{$^2$ Harbin Institute of Technology, China.}
\thanks{$\dagger$ Correspondence: zhongqiang.ren@sjtu.edu.cn}
}
\maketitle


\begin{abstract}
\input{abstract}
\end{abstract}



\section{Introduction}
\label{sec:intro}
\input{introduction}


\section{Problem Formulation}\label{cbssShape:sec:problem}
\input{problem_def}

\section{Preliminaries}
\label{sec:preli}
\input{preliminary}

\section{Method}
\label{sec:method}
\input{method}

\section{Experimental Results}
\label{sec:result}
\input{results}

\section{Conclusions and Future Work}
\label{sec:conclusion}
\input{conclusion}

\bibliographystyle{IEEEtran}
\bibliography{references}



\end{document}

%% file: abstract.tex
This paper investigates Multi-Agent Path Finding with Elevators (MAPF-E), which seeks conflict-free paths for multiple agents whose start and goal locations may be on different floors, and the agents can use elevators to travel between floors.
The existence of elevators complicates the interaction among the agents and introduces new planning challenges.
On the one hand, elevators can cause many conflicts among agents due to their relatively long traversal time across floors, especially when many agents need to reach a different floor.
On the other hand, the planner has to reason in a larger state space that includes elevator states in addition to agent locations.
To address these challenges, this paper proposes CBS-E, an extension of Conflict-Based Search (CBS) to solve MAPF-E optimally.
CBS-E introduces new concepts of elevator constraints to incorporate the state of the elevator into the conflict resolution process of CBS.
We also extend Multi-Valued Decision Diagrams (MDDs) to CBS-E, referred to as MDD-E. 
MDD-E enables the intelligent selection of conflicts to resolve in each iteration, thereby improving the runtime efficiency.
The results show that our CBS-E often doubles or triples the success rates of the baseline, and the proposed conflict reasoning can reduce the number of iterations of CBS-E by up to an order of magnitude.


%% file: introduction.tex
Multi-Agent Path Finding (MAPF) seeks conflict-free paths for multiple agents from their respective starts to goals within a shared environment, while minimizing the path cost.
This paper considers a variant of MAPF where the environment consists of multiple floors connected by elevators, and each elevator can transport one agent at a time.
The start and goal locations of the agents may lie on different floors, and the agents use elevators to move between floors.
We refer to this variant as Multi-Agent Path Finding with Elevators (MAPF-E).
When there is only one floor, MAPF-E becomes the regular MAPF.
MAPF is NP-hard~\cite{yu2013structure_nphard} and so is MAPF-E.
MAPF-E naturally arises in logistics when the environment consists of multiple floors.
Consider a 3D warehouse, where mobile robots must use elevators to reach goal locations for retrieving or placing items (Fig.~\ref{fig:background1}).


\begin{figure}[tb]
    \centering
    \includegraphics[width=1\linewidth]{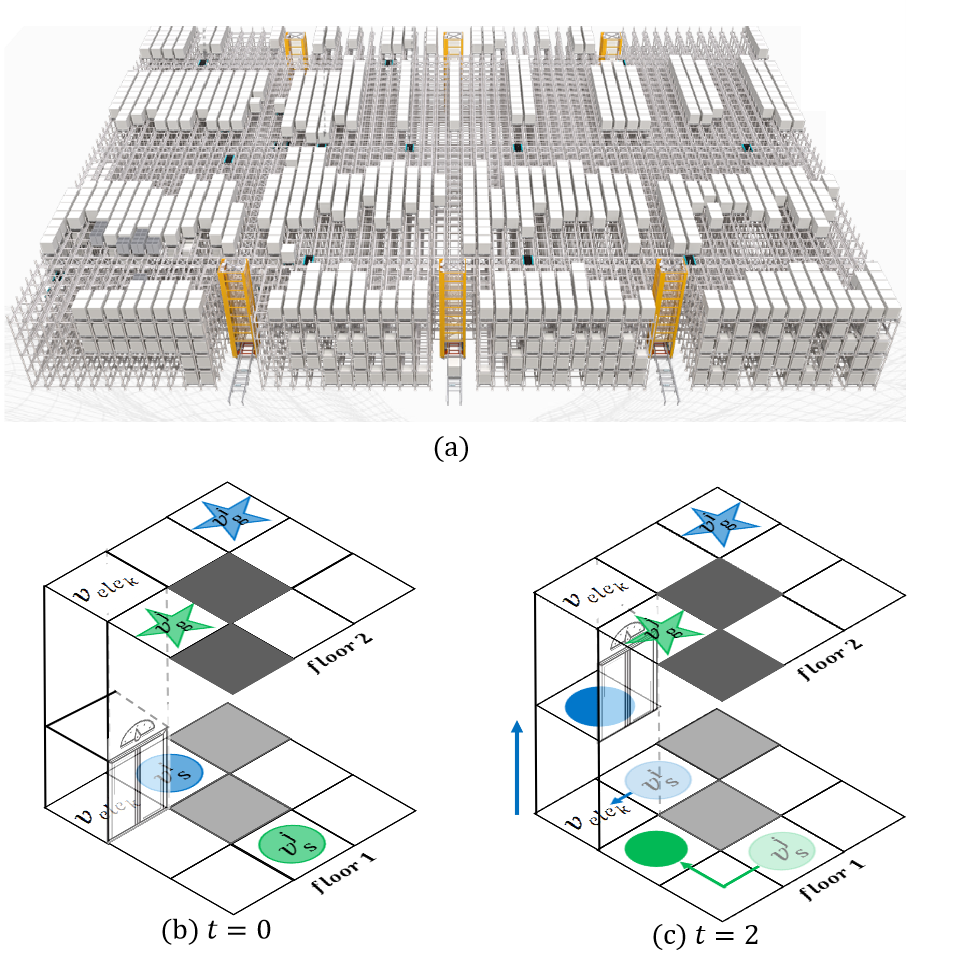}
        \vspace{-5mm}
    \caption{(a) The motivation of MAPF-E. Robots move in a 3D warehouse, while inter-layer transitions are enabled by the elevators (the orange structures). The picture is from ZS Robotics~\cite{zsrobotics2025}.
    (b) A toy example of MAPF-E, which seeks start-goal paths for the agents. All elevator edge costs are 2, and the black squares represent obstacles. (c) At time $t=2$, agent $i$ is using the elevator, and agent $j$ needs to wait at a vertex outside the elevator. 
    }
    \vspace{-3mm}
     \label{fig:background1}
\end{figure}



Although a variety of planners and CBS enhancements for MAPF and its variants have been developed, ranging from optimal planners~\cite{sharon2015conflict,wagner2015subdimensional,boyarski2015icbs}, bounded sub-optimal planners~\cite{barer2014suboptimal,li2021eecbs} to unbounded sub-optimal planners~\cite{okumura2022priority,de2013push}, we are not aware of any research focusing on MAPF-E.
A recent work~\cite{Wang20243DWarehouse} considers MAPF in a 3D space in which agents can move vertically at arbitrary locations, which is different from the elevator-based setting considered in this work.

Existing approaches for MAPF can be adapted to solve MAPF-E, by describing all floors as a graph with special edges representing the elevator between floors and running a MAPF planner on this graph.
However, such a naive adaptation can lead to poor runtime efficiency.
The existence of elevators complicates the interaction among the agents and introduces new challenges for planning.
First, an elevator can couple multiple agents on different floors, even though these agents are spatially far away from each other, thereby complicating the collision resolution process.
Second, it often takes much longer for an elevator to traverse between adjacent floors than for an agent to move from one location to another on the same floor, which can lead to frequent collisions and congestion among the agents, especially when many agents seek to take the same elevator.
In this case, the planner has to consider the possibility of letting an agent take a long detour to use another elevator to optimize the arrival time.
Consequently, the planner must reason over a larger state space that explicitly incorporates elevator states and positions, increasing the computational burden.

As a first attempt to study MAPF-E, this paper focuses on developing exact algorithms that can find optimal solutions for MAPF-E.
Based on Conflict-Based Search (CBS)~\cite{sharon_conflict-based_2015}, a popular approach for MAPF, we develop CBS-E and its variants to optimally solve MAPF-E.
In particular, we find that, a naive adaptation of CBS for MAPF-E may need an exponentially growing number of iterations to resolve agent-agent conflicts as the traversal duration of the elevator per layer increases.
To address this difficulty, our CBS-E can often resolve an elevator conflict in a single iteration, by introducing a new approach for elevator constraints generation based on the state of the elevator.
We prove that, with these new constraints, our CBS-E preserves the solution optimality.
Additionally, we also develop new conflict reasoning techniques by extending Multi-Valued Decision Diagram (MDD), an important data structure for conflict reasoning in CBS~\cite{sharon2012meta}, to MDD with elevator (MDD-E), enabling CBS-E to intelligently select conflicts to resolve in each iteration and further enhance runtime efficiency.

We conduct an ablation study of the proposed conflict reasoning techniques in CBS-E, and compare our CBS-E against a naive adaptation of CBS to MAPF-E.
Experimental results show that CBS-E often doubles or triples the success rate of the naive CBS adaptation in various maps, and the proposed conflict reasoning techniques can reduce the number of iterations of CBS-E by up to an order of magnitude.



%% file: problem_def.tex

Let the index set \( I = \{1, 2, \dots, N\} \) represent a set of $N$ agents.
All agents operate within a shared workspace, modeled as a finite graph \( G = (V, E) \), which consists of $m$ subgraphs $G=G_1\cup G_2\cup \cdots \cup G_m$.
Each subgraph $G_l,l\in L=\{1,2,\cdots,m\}$ represents the $l$-th floor (or layer).
The subgraphs $G_l$ may differ from one another.
The vertex set $V$ represents all possible locations of the agents.
The edge set $E \subseteq V\times V$ represents the transitions between vertices in $V$.
Each edge $e$ between $u,v \in V$  has a positive integer cost value $c(e)=c(u,v)\in \mathbb{Z}^+$, indicating the traversal time of any agent along that edge.

The vertices are classified into two categories $V=V_{ele} \cup V_{reg}$, where $V_{ele}$ denotes the set of vertices corresponding to elevators, and $V_{reg}$ are all other regular vertices.
The edges are also classified into two categories $E=E_{ele}\cup E_{reg}$ as follows.
An edge $e=(u,v)\in E_{ele}$ is an elevator edge if both ends of $e$ are elevator vertices $u,v \in V_{ele}$.
Otherwise, an edge is a regular edge.
Let the index set $I_K=\{1,2,\cdots,K\}$ denote a set of $K$ elevators.
Each elevator $k\in I_K$ corresponds to a set of vertices $V_{ele,k}\subset V_{ele}$ and a set of edges $E_{ele,k} \subseteq E_{ele}$.
Each vertex $v_{k,l} \in V_{ele,k}$ corresponds to a layer $G_l$ and any two vertices $v_{k,l_1},v_{k,l_2} \in V_{ele,k}$ cannot share the same layer $l_1\neq l_2$.
Each edge $e \in E_{ele,k}$ connects two adjacent layers.
Each regular edge $e\in E_{reg}$ has a unit-cost value $c(e)=1$, while each elevator edge can have various known and fixed integer costs.
An elevator can hold at most one agent at a time, and each agent enters an elevator at most once\footnote{This implies every agent can take an elevator to directly reach its goal floor, without taking any elevator for a second time.}.

Let $v_s^i, v_g^i \in V$ denote the start and goal vertices of agent $i \in I $, and $l_s^i, l_g^i$ denote the floors on which $v_s^i, v_g^i$ are located, respectively.
All agents share a global clock and start moving along their paths from $v_s^i$ at $t=0$. Let $\pi^i(v^i_s,v^i_g)=(v_s^i,v_1^i,v_2^i, \dots, v_g^i)$ denote a path for agent $i$ between vertices $v_s^i$ and $v_g^i$.
The time steps $t(v^i_1),t(v^i_2)$ between two adjacent elevator vertices $v^i_{1}$ and $v^i_{2}$ in $\pi^i$ are equal to the elevator edge cost, i.e., $|t(v^i_2)-t(v^i_1)| = c(e), e=(v^i_{1},v^i_{2})\in E_{ele}$.
For all other adjacent vertices \(v_1^i, v_2^i\) in \(\pi^i\), the time steps are one time unit, i.e., the cost of a regular edge $c(v^i_1,v^i_2)$.
Let $g(\pi^i(v^i_s,v^i_g))$ denote the cost of a single path, which is defined as the sum of costs of the edges along the path.
Let $\pi = (\pi^{1}, \pi^{2}, \cdots, \pi^{N})$ denote a joint path of all the agents, where $\pi^{i}$ is the path of agent $i$. The total cost of the joint path is defined as the sum of the individual costs, i.e., $g(\pi) = \sum_{i} g(\pi^{i})$.

A conflict between two agents $i,j\in I$ occurs if any of the following three cases arises. 
The first two cases are the same as in MAPF and the third case is new in MAPF-E:
(1) A vertex conflict $(i,j,v,t)$ occurs when two agents $i,j\in I$ occupy the same vertex $v$ at the same time step $t$; (2) An edge conflict $(i,j,e,t)$ occurs when two agents $i,j\in I$ go through the same edge $e$ from opposite directions between times $t$ and $t+1$; (3) An \emph{elevator conflict}, as defined below.

\begin{definition}[Elevator Conflict]\label{def:ele-conf}
For an agent $i$, let $t^i_s, t^i_g$ denote the time steps when $i$ enters and leaves an elevator $k\in I_K$ to travel from $v_{k,l_1}$ to $v_{k,l_2}$ (i.e., $v_{k,l_1} \;\text{and}\; v_{k,l_2}\in V_{ele,k}$), respectively.
Consider another agent $j$ who travels from $v_{k,l_3} \in V_{ele,k}$ after $i$ using the same elevator.
We identify three cases for $j$ concerning elevator $k$: 
\begin{itemize}
	\item The elevator is \emph{occupied} during the interval $[t^i_s, t^i_g]$ because it is transporting agent $i$;
	\item The elevator is \emph{resetting} during the interval $[t^i_g, t^i_g+\triangle]$, where $\triangle=c((v_{k,l_2}, v_{k,l_3}))$, 
		since $k$ needs to move from $v_{k,l_2}$ to $v_{k,l_3}$ before picking up agent $j$;
	\item The elevator is \emph{available} after $t^i_g+\triangle$ because it has moved to the required floor after transporting $i$; 
\end{itemize}
An elevator conflict occurs when agent $j$ arrives at any vertex in the set $V_{ele,k}$ related to the elevator $k\in I_K$ while that elevator is \emph{resetting} or \emph{occupied} (Fig.~\ref{fig:elevatorcon}). In this work, we assume all elevators are \emph{available} to all agents at $t=0$.
\end{definition}

The goal of MAPF-E is to find a conflict-free joint path $\pi$ with the minimal cost.

\begin{figure}[tb]
    \centering
    \includegraphics[width=1\linewidth]{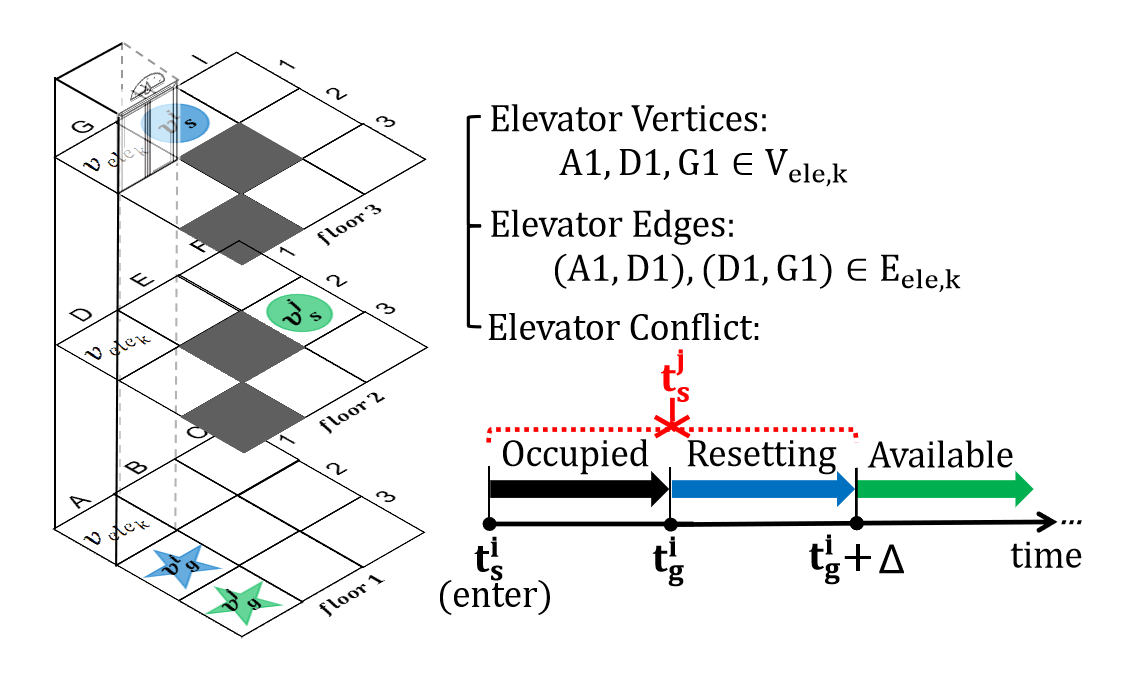}
        \vspace{-5mm}
            \caption{Agent $i$  (blue) and $j$ (green) need to use the same elevator $k$ to reach their goals. An elevator conflict arises when agent $j$ attempts to enter the elevator with the \textit{occupied} or \textit{resetting} state, i.e., when $t_s^j \in [t_s^i,\, t_g^i + \triangle]$.}
     \vspace{-3mm}
     \label{fig:elevatorcon}
\end{figure}


%% file: preliminary.tex
\newcommand{\procName}[1]{\textsc{#1}}
\subsection{Conflict-Based Search}
Conflict-Based Search (CBS)~\cite{sharon_conflict-based_2015} is a two-level search algorithm that finds an optimal joint path for all agents for MAPF.
At the high-level, CBS searches by building a binary tree, called the \emph{constraint tree} (CT).
Each node in the CT represents a set of constraints $\Omega$, a joint path $\pi$ satisfying the constraints in $\Omega$, and its $g$-value, the cost of $\pi$.

CBS begins by constructing a root node $R_{\text{root}} = (\pi_0, g_0, \Omega_0)$, where $\pi_0 = (\pi^1, \pi^2, \dots, \pi^N)$ is the shortest path for each agent ignoring any agent-agent conflicts and $\Omega_0=\emptyset$. 
At any time during the search, all leaf nodes of the CT form an open list, denoted as OPEN, which is a priority queue sorted in ascending order of $g$-value.
CBS repeatedly selects the node with the smallest $g$-value from OPEN for expansion.
CBS expands a node $P = (\pi, g, \Omega)$ by checking for a conflict among the agent paths in $\pi$.
If no conflict is detected, the current joint path $\pi$ is returned as an optimal solution.
Otherwise, for a detected conflict $C = (i,j,v,t)$, two new constraints, $(i,v,t)$ and $(j,v,t)$, are generated.
Each constraint leads to a new constraint set, $\Omega \cup \{(i,v,t)\}$ and $\Omega \cup \{(j,v,t)\}$, respectively (the same applies to edge conflicts $(i,j,e,t)$).
For each new constraint set $\Omega'$, a low-level search is performed to compute a new optimal path for the affected agent $i$ that satisfies all the constraints in $\Omega'$. 
The new joint path $\pi'$ is formed by replacing the old path of agent $i$ while keeping the paths of the other agents unchanged.
A corresponding node $P' = (\pi', g', \Omega')$ is created and added to OPEN for future expansion.
The search continues until a node with conflict-free joint paths is found.
CBS guarantees that the returned solution has the minimum cost.

\subsection{MDD, Bypassing Conflicts and Conflict Selection}
At the high-level of CBS, for a selected node $P=(\pi,g,\Omega)$ from OPEN, there may be multiple conflicts in $\pi$. Both the choice of which conflict to resolve and the method used to select an appropriate resolution can significantly affect the runtime efficiency of CBS~\cite{boyarski2015icbs}. To optimize this process, we employ two techniques, namely \emph{Bypassing Conflict} and \emph{Conflict Selection}, both of which are based on Multi-Valued Decision Diagrams (MDD)~\cite{guni2013ict}.

\subsubsection{Multi-Valued Decision Diagrams (MDD)}

The MDD$^i_d$ for agent \(i\) in a CT node \(P\) is defined as a directed graph representing all optimal paths (each with cost \(d\)) that are subject to the set of constraints $\Omega$ of $P$.
 Each node in the MDD corresponds to a tuple $(v,t)$, where $v$ is a vertex reached by agent $i$ at time $t$ on an optimal path, and it appears at time $t$ in MDD$^i_d$.
When an MDD contains only a single node $(v,t)$ at time $t$, and a new constraint prohibits the agent from occupying $v$ at that time, all optimal paths (with cost $d$) are eliminated.
The agent with this new constraint must take a new path with path cost $d+1$, thereby increasing the path cost.
This property allows MDDs to classify a conflict by checking whether a constraint eliminates all paths in MDD$^i_d$.
Specifically, this check is performed by constructing the joint MDD$^{ij}$ for agents $i$ and $j$. 
A node in the joint MDD$^{ij}$ is represented as $((v_i,t), (v_j,t))$, indicating that there exists a pair of conflict-free paths for agents $i$ and $j$, with the respective optimal paths passing through $(v_i,t)$ and $(v_j,t)$.
Fig.~\ref{fig:traMDD} shows an example.


\subsubsection{Bypassing Conflicts (BP)}

 For a popped node $P=(\pi,g,\Omega)$ containing a conflict $C$ in $\pi$, sometimes it is beneficial to directly modify the joint path $\pi$ in $P$ to avoid branching that would generate two new nodes and increase the size of the CT.
This technique is known as bypass conflicts (BP).
Specifically, a path $\pi^{i}_{new}$ is a \textit{valid bypass} to path $\pi^{i}$ for agent $i$ with respect to a conflict $C$ and the CT node $P$, if the following conditions are satisfied: (i) $g(\pi^{i}_{new})=g(\pi^{i})$, (ii) both $\pi^{i}_{new}$ and $\pi^{i}$ satisfy the constraints $\Omega$ of $P$, (iii) $\pi^{i}_{new}$ does not violate the conflict $C$.
Since the joint MDD$^{ij}$ contains all optimal conflict-free paths between agents $i$ and $j$, $\pi^{i}_{new}$ can be extracted from the joint MDD$^{ij}$ if it exists. Then, $P$ adopts $\pi^{i}_{new}$ directly and is added back to OPEN.

\subsubsection{Conflict Selection}
All conflicts can be classified into three types~\cite{boyarski2015icbs}:
(1) \textit{Cardinal conflict}: A conflict $C = (i,j,v,t)$ is cardinal for a CT node $P$ if neither agent $i$ nor agent $j$ can find a valid bypass in the joint MDD$^{ij}$. This implies that each of the two constraints derived from $C$ increases the cost $g$ of the CT node $P$.
(2) \textit{Semi-cardinal conflict}: A conflict $C$ is semi-cardinal if only one of the agents, either $i$ or $j$, can find a valid bypass in the joint MDD$^{ij}$. This implies that only one of the constraints derived from $C$ increases the cost $g$ of the CT node $P$.
(3) \textit{Non-cardinal conflict}: A conflict $C$ is non-cardinal if both agents $i$ and $j$ can find a valid bypass in the joint MDD$^{ij}$. This implies that neither of the two constraints derived from $C$ increases the cost $g$ of the CT node $P$.
Since the overall process of CBS aims to increase $g$ toward its final optimal value, we prefer to resolve conflicts that cause $g$ to increase more rapidly.
Therefore, CBS often selects and resolves conflicts in the following order~\cite{boyarski2015icbs}: {cardinal conflicts}, {semi-cardinal conflicts}, and then {non-cardinal conflicts}.

\begin{figure}[tb]
    \centering
    \includegraphics[width=\linewidth]{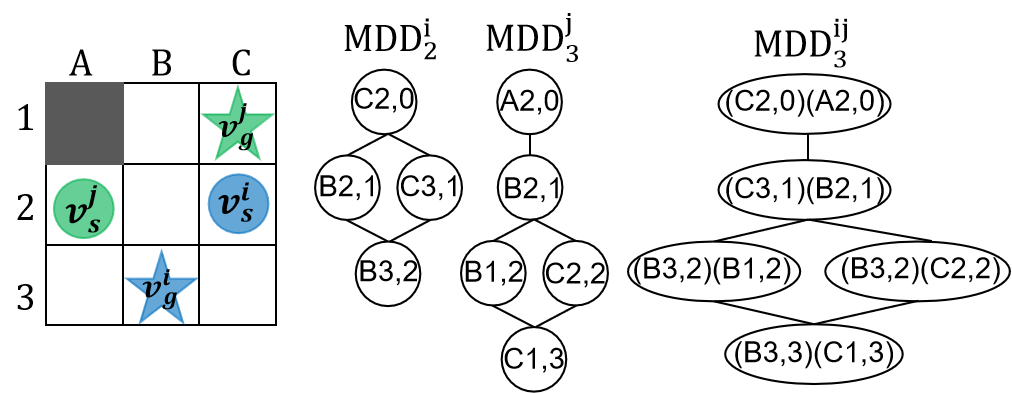}
        \vspace{-5mm}
    \caption{Each node $(v,t)$ in the MDD means the agent can optimally reach the goal via vertex $v$ at time $t$. For instance, in $MDD^i_2$, at $t=1$, the agent can reach either $B2$ or $C1$, resulting in two nodes at this layer. In the joint MDD, each node must correspond to a conflict-free partial path. For instance, in $MDD^{ij}_3$ at $t=1$, a conflict would occur if agents $i$ and $j$ both visited $B2$ simultaneously. So the joint MDD at $t=1$ contains only the node $((C3,1),(B2,1))$.
}
    \vspace{-3mm}
     \label{fig:traMDD}
\end{figure}

\begin{algorithm}[tb]
    \small
	\caption{CBS-E}
    \label{alg:newCBS}
	\begin{algorithmic}[1]
        \small
            \Statex{INPUT: $G=(V,E,C)$}
            \Statex{OUTPUT: a conflict-free joint path $\pi$ in $G$}
		\State{$\Omega_0 \gets \emptyset$, $\pi_0 \gets$ \procName{LowLevel}($i,\Omega_{}$)}, $\forall i\in I$, $g_0 \gets g(\pi_0)$
		\State{Add $R_{root}=(\pi_0,g_0,\Omega_0)$ to OPEN}\label{cbssapx:alg:cbssa:lineInitEnd}
		\While{OPEN $\neq \emptyset$} 
         \State $P=(\pi,g,\Omega) \gets$ best node from \textsc{OPEN}
    \If{$\pi$ has no conflict}
        \State \Return $\pi$ 
    \EndIf
    \State \underline{$C \gets$ \procName{FindConflict}($\pi$)}\label{cbssapx:alg:cbssa:MDDE}
    \If{$C$ is not cardinal}
        \If{\procName{FindBypass}($P$, $C$)} 
            \State \textbf{continue}
        \EndIf
    \EndIf
    \ForAll{agent $i$ in $C$}
    \State \underline{$\Omega' \gets \Omega \cup \procName{NewConstraints}(i,C)$}\label{cbssapx:alg:cbssa:EC}
    \State{$\pi'_{},g'_{} \gets$ \procName{LowLevel}($i,\Omega'_{}$)}
    	\State{Add $P'=(\pi',g',\Omega')$ to OPEN}
		
    \EndFor
		\EndWhile \label{}
		\State{\textbf{return} failure}
	\end{algorithmic}
\end{algorithm}

%% file: method.tex
CBS-E (Alg.~\ref{alg:newCBS}) follows a process similar to that of CBS, with the primary differences occurring on Line~\ref{cbssapx:alg:cbssa:MDDE} and~\ref{cbssapx:alg:cbssa:EC}, which represent our contributions.
On Line~\ref{cbssapx:alg:cbssa:MDDE}, we introduce MDD-E to classify and select conflicts.
On Line~\ref{cbssapx:alg:cbssa:EC}, we introduce a new type of elevator constraint to efficiently handle elevator conflicts. 

\subsection{Elevator Constraints (EC)}

The conventional conflict resolution technique in CBS can also be applied to resolve elevator conflicts.
For two agents, \(i\) and \(j\), that are in conflict, each iteration of CBS forbids either agent \(i\) or \(j\) from occupying a vertex at a specific time step \(t\).
However, this approach is inefficient for resolving an elevator conflict since many time steps must be forbidden for an agent.
Fig.~\ref{fig:Inefficiency}(b) provides such a toy example, where a large number of CT nodes are generated when resolving the elevator conflict using conventional CBS.
By contrast, the same elevator conflict can be resolved in a single iteration (Fig.~\ref{fig:Inefficiency}(c)) using our new method as explained below.

\begin{figure*}[tb]
    \centering
    \includegraphics[width=0.9\linewidth]{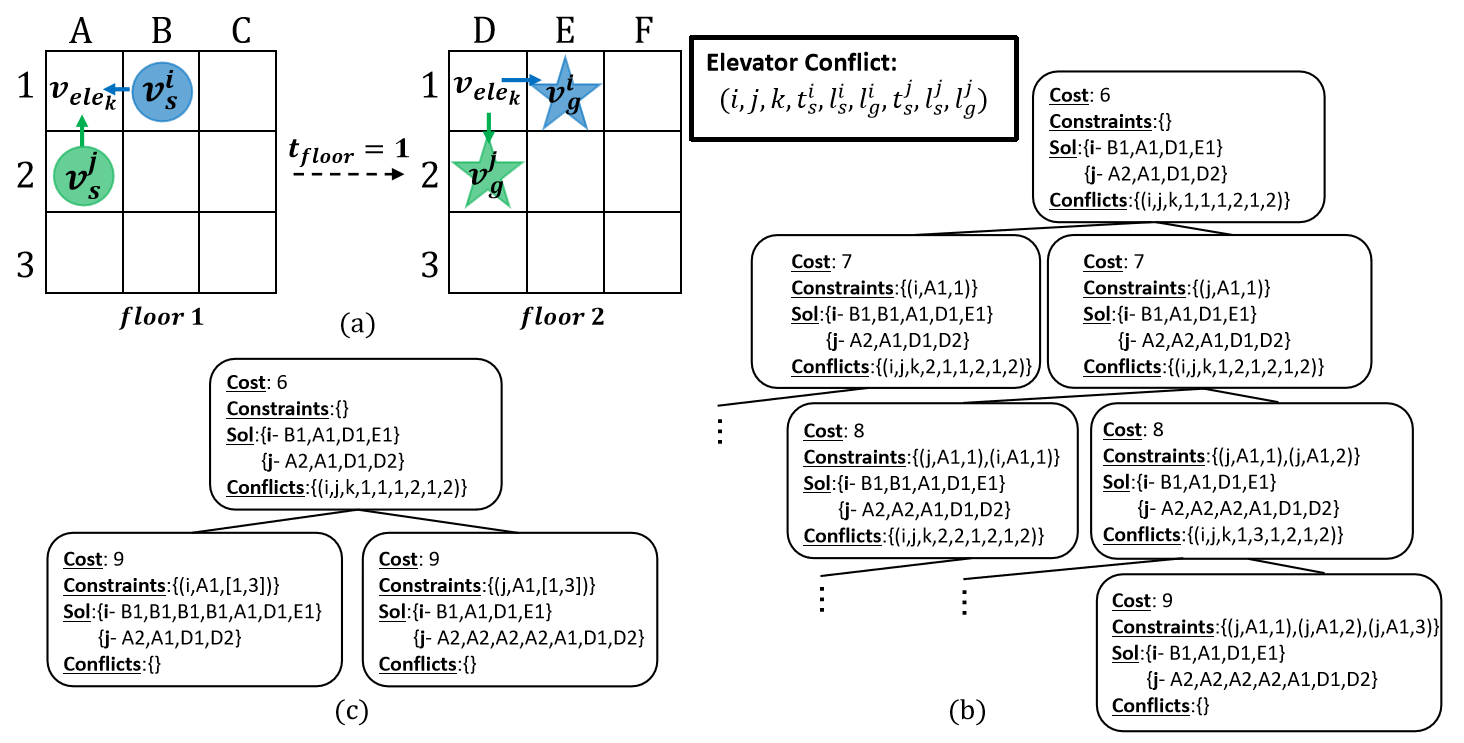}
        \vspace{-5mm}
    \caption{(a) A toy example of MAPF-E, which seeks start-goal paths for the agents. The colors indicate the agents. The traversal time of any elevator edge is $t_{floor} = 1$.
(b) Inefficiency of resolving the elevator conflict by single-timestep constraint in CBS, which requires expanding the CT at least \(\min\{(t^j_s+t^j_o+t^j_r-t^i_s+1),\ (t^i_s+t^i_o+t^i_r-t^j_s+1)\}\) times. (c) With EC, the elevator conflict can be resolved with a single expansion.
}

     \label{fig:Inefficiency}
\end{figure*}

Let \((i,j,k,t_s^i,l_s^i,l_g^i,t_s^j,l_s^j,l_g^j)\) denote an elevator conflict, indicating that agents \(i\) and \(j\) encounter a conflict involving elevator \(k\).
Here, agent $i$, with start and goal vertices located at floor $l_s^i,l_g^i$ respectively, starts to use elevator $k$ at time $t^i_s$, and agent $j$, with start and goal vertices located at floor $l_s^j,l_g^j$ respectively, starts to use elevator $k$ at time $t^j_s$.

To simplify the presentation, we assume that the time required for an elevator to traverse any elevator edge between two adjacent layers is constant. Let $t_{floor} = c(e), e \in E_{\text{ele},k}$ denote the cost of any elevator edge.
Then, the duration for an elevator to transport agent $i$ from layer $l_s^i$ to $l_g^i$ is $t^i_o = {|l_s^i - l_g^i| \cdot {t_{floor}}}$. 
Before the elevator can serve another agent $j\neq i$, the elevator must move to the start floor of agent $j$ incurring a resetting cost of $t^i_r = |l_g^i - l_s^j| \cdot {t_{floor}}$.
Similarly, the corresponding parameters for agent $j$, such as $t^j_s$ (the time when agent $j$ enters this elevator), $t^j_o$ (duration that the elevator is occupied by $j$ during transportation) and $t^j_r$ (time for the elevator to return to agent $i$'s floor after transporting agent $j$) can be computed.
Thus, due to agent $i$'s usage of the elevator, the time interval during which the elevator $k$ cannot be used by agent $j$ is at least ${IN}_i$(\ref{CBS-E:eqn:inter1}) and ${IN}_j$(\ref{CBS-E:eqn:inter2}) can be similarly defined.   
\begin{align}
 {IN}_i = [t^i_s, \, t^i_s + t^i_o + t^i_r] \label{CBS-E:eqn:inter1}\\ 
  {IN}_j = [t^j_s, \, t^j_s + t^j_o + t^j_r]  \label{CBS-E:eqn:inter2}
\end{align}
When an elevator conflict is first detected,
without loss of generality, we consider the case where $t^i_s <  t^j_s$, i.e., agent $i$ starts using elevator $k$ earlier than agent $j$, and that $t^j_s \leq t^i_s+t^i_o+t^i_r$, indicating that the two agents experience an elevator conflict over elevator $k$.

To resolve such an elevator conflict, the following two constraint sets are generated.

\vspace{-5mm}
\begin{align}
  \Omega_i=\{\omega^i=(i,v_{k,l^i_s},t^i)| t^i \in [t^i_s,t^j_s+t^j_o+t^j_r] \} \label{CBS-E:eqn:cstr1}\\ 
  \Omega_j=\{\omega^j=(j,v_{k,l^j_s},t^j)| t^j \in [t^j_s,t^i_s+t^i_o+t^i_r] \},\label{CBS-E:eqn:cstr2}
\end{align}
where $v_{k,l^i_s}$ and $v_{k,l^j_s}$ denote the entrances of elevator $k$ at floors $l^i_s$ and $l^j_s$, respectively.
The first set of constraints prohibits agent~$i$ from entering elevator~$k$ from vertex $v_{k,l^i_s}$ during the time interval $[t^i_s,t^j_s+t^j_o+t^j_r]$.  
The second set of constraints prohibits agent~$j$ from entering elevator~$k$ from vertex $v_{k,l^j_s}$ during the time interval $[t^j_s,t^i_s+t^i_o+t^i_r]$.

As shown in Fig.~\ref{fig:Inefficiency}(c), by using EC, the CT needs to branch only once to resolve the elevator conflict, thereby addressing the long-term coupling between agents $i$ and $j$ caused by the elevator.
In this example, $t^i_s = 1$, $t^j_s=1$, $t^j_o=1$, $t^j_r=1$. As a result, the set of constraints for agent $i$ is $\Omega_i=\{i,A1,[1,3]\}$. Similarly the set of constraints for agent $j$ is $\Omega_j=\{j,A1,[1,3]\}$.

To prove CBS with this new branching rule returns an optimal solution, we need to show the two constraint sets are \emph{mutually disjunctive}~\cite{li2019multi}. Specifically, two sets $\Omega_i$ and $\Omega_j$ are said to be mutually disjunctive if no conflict-free joint path exists that can simultaneously violate any pair of constraint sets $(\omega^i, \omega^j)$ with $\omega^i \in \Omega_i$ and $\omega^j \in \Omega_j$.
In other words, if paths $\pi^{i}$ and $\pi^{j}$ simultaneously violate constraint sets (\ref{CBS-E:eqn:cstr1}) and (\ref{CBS-E:eqn:cstr2}), then there must exist an elevator conflict.

\begin{theorem}
\label{theorem:1}
Constraints (\ref{CBS-E:eqn:cstr1}) and (\ref{CBS-E:eqn:cstr2}) are mutually disjunctive.
\end{theorem}
\begin{proof}
Suppose agent~$i$ (entering elevator $k$ from vertex $v_{k,l^i_s}$ at $t^i_s$) and agent~$j$ (entering elevator $k$ from vertex $v_{k,l^j_s}$ at $t^j_s$) encounter their first elevator conflict at elevator $k$, agents $i$ and $j$ are then assigned two constraints $\Omega_i$ and $\Omega_j$. Assume that in the new paths of agents $i$ and $j$, agent~$i$ enters elevator $k$ at time $x$ such that elevator $k$ cannot be used by agent~$j$ for at least the interval $I_i =[x, \, x + t^i_o + t^i_r]$ by Eq.\ref{CBS-E:eqn:inter1}, and agent~$j$ enters elevator $k$ at time $y$ such that elevator $k$ cannot be used by agent~$i$ for at least the interval $I_j=[y, \, y + t^j_o + t^j_r]$ by Eq.\ref{CBS-E:eqn:inter2}.
We need to show that if the new paths of agents $i$ and $j$ violate the constraints $\Omega_i$ and $\Omega_j$, for any pair $(x,y)$ with $x \in [t^i_s,\,t^j_s+t^j_o+t^j_r]$ and $y \in [t^j_s,\,t^i_s+t^i_o+t^i_r]$, the two corresponding intervals $I_i$ and $I_j$ must intersect, which indicates a conflict.
We show the following cases cannot happen by contradiction: 
    case 1: $y > x+t^i_o+t^i_r$, and case 2: $x > y+t^j_o+t^j_r$.
For the case 1, since $x \in [t^i_s,\,t^j_s+t^j_o+t^j_r]$ and $y \in [t^j_s,\,t^i_s+t^i_o+t^i_r]$, we have $
y \leq t^i_s+t^i_o+t^i_r$ and  $x+t^i_o+t^i_r \geq t^i_s+t^i_o+t^i_r$. 
Thus, $y \leq t^i_s+t^i_o+t^i_r \leq x+t^i_o+t^i_r$, which contradicts the assumption $y > x+t^i_o+t^i_r$.
For the case $x > y+t^j_o+t^j_r$, since $x \in [t^i_s,\,t^j_s+t^j_o+t^j_r]$ and $y \in [t^j_s,\,t^i_s+t^i_o+t^i_r]$, we have $x \leq t^j_s+t^j_o+t^j_r$, and $y+t^j_o+t^j_r \geq t^j_s+t^j_o+t^j_r$.
Thus, $x \leq t^j_s+t^j_o+t^j_r \leq y+t^j_o+t^j_r$, which contradicts the assumption $x > y+t^j_o+t^j_r$.
Therefore, if a joint path $\pi$ simultaneously violates any pair of constraints  $(\omega^i,\omega^j), \omega^i\in \Omega_i, \omega^j \in \Omega_j$, $\pi$ must contain elevator conflicts, and constraints~(\ref{CBS-E:eqn:cstr1}) and~(\ref{CBS-E:eqn:cstr2}) are mutually disjunctive. 
\end{proof}

\begin{remark}\label{rmk:corrdor}
The conventional CBS is known to be inefficient when resolving conflict in long corridor~\cite{sharon2015conflict}, a problem that is addressed by \emph{corridor reasoning}~\cite{li_pairwise_2021}.  
The key difference between corridor conflicts and elevator conflicts lies in the agents' behavior: in a feasible solution, agents can wait or avoid other agents by temporarily entering, whereas an elevator can only transport agents between floors and can be unavailable even when no agents are using it (due to resetting).
As a result, \emph{corridor reasoning} is not applicable to elevator conflicts.
\end{remark}

\begin{figure}[tb]   
    \centering
    \includegraphics[width=0.5\textwidth]{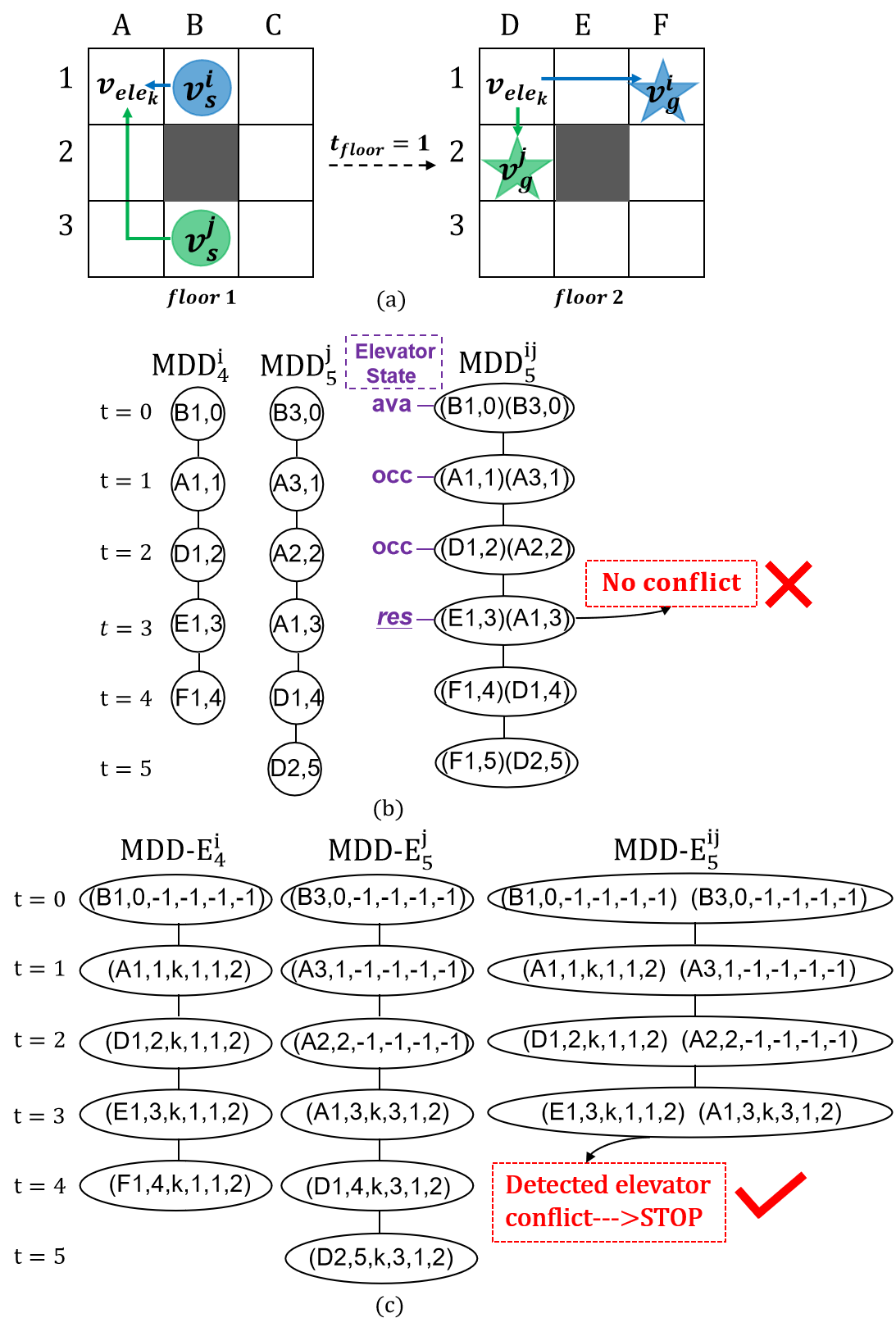} 
    \vspace{-6mm}
    \caption{(a) A toy example of MAPF-E, where agents (green and blue) move from start to goal, black blocks denote obstacles, and each elevator edge has a traversal time of $t_{floor}=1$. (b) shows that when constructing joint MDD by the classical method, an elevator conflict occurring at $t=3$ cannot be detected, because the nodes in the traditional MDD do not contain elevator-related information and cannot capture the elevator state; (c) shows that with MDD-E, the elevator conflict can be successfully detected during the construction of the joint MDD-E.
    }

    \vspace{-3mm}
    \label{fig:classicMDD}
\end{figure}
\subsection{MDD with Elevator Details \texorpdfstring{(MDD-E)}{(MDDE)}}

Multi-Valued Decision Diagrams (MDD)~\cite{sharon2012meta} cannot be directly applied to MAPF-E as it does not consider the state of the elevator and therefore fails to detect elevator-related conflicts.
As illustrated in Fig.~\ref{fig:classicMDD}(b), when constructing the joint MDD using the classical method, no edge or vertex conflict is detected between agent $i$ and $j$ at $t=3$. Consequently, the node $((E1,3),(A1,3))$ is added to $MDD^{ij}_5$. Nevertheless, an elevator conflict actually occurs at $t=3$: the elevator is in the resetting state due to agent~$i$'s usage, while agent~$j$ attempts to use it. Since elevator-related information is not encoded in the nodes, this conflict remains undetected during the construction of $MDD^{ij}_5$. We therefore modify MDD to incorporate elevator states.

Let $(v,t,k,t^i_s,l_s^i,l_g^i)$ denote an MDD-E node, where agent $i$ has start and goal vertices located on layers $l_s^i$ and $l_g^i$, respectively, and enters elevator $k$ at time $t^i_s$. Let $k,t^i_s,l^i_s,l^i_g=-1$ to indicate that agent~$i$ has not yet used any elevator. Each node at time $t$ in {MDD-E}$^{i}_d$ corresponds to a possible location of agent $i$ at time $t$ along a path of cost $d$ from $v_s^i$ to $v_g^i$.
Edges between nodes in {MDD-E}$^{i}_d$ are directed, representing feasible transitions where agent $i$ can move from one node to another in the next time step. 

When constructing the joint {MDD-E}$^{ij}_d$ for agents \(i\) and \(j\),  
if the corresponding nodes in {MDD-E}$^{i}_d$ and {MDD-E}$^{j}_d$ have the same \(k\) value of elevator and \(k \neq -1\), then the other information in the nodes is used to determine whether an elevator conflict occurs, according to the detection method described in Sec.~\ref{cbssShape:sec:problem} (Example~\ref{eg:MDD-E}).  
Specifically, an elevator conflict exists if either (\ref{CBS-E:eqn:con1}) or (\ref{CBS-E:eqn:con2}) holds:
\begin{align}
  t^i_s &\in [t^j_s,\, t^j_s + t^j_o + t^j_r] \label{CBS-E:eqn:con1}\\
  t^j_s &\in [t^i_s,\, t^i_s + t^i_o + t^i_r] \label{CBS-E:eqn:con2}
\end{align}
By identifying the types of elevator conflicts using MDD-E and integrating Conflict Selection and Bypass techniques at the high-level, the CBS with MDD-E can compute a lower bound on the joint path cost that is at least as tight as that of CBS without MDD-E, given the same number of CT nodes.
\begin{example}\label{eg:MDD-E}
Using our MDD-E, each node is augmented with elevator-related information, allowing the joint MDD-E to correctly detect elevator conflicts. As shown in Fig.~\ref{fig:classicMDD}(c), node $(E1,3,k,1,1,2)$ of {MDD-E}$^{i}_4$ encodes that agent~$i$ starts using elevator $k$ at $t=1$, traveling from floor~1 to floor~2. Here, elevator $k$ is still in the resetting state at $t=3$.
Since agent~$j$ is not allowed to use an elevator that is in the resetting state, the elevator conflict can be correctly detected during the construction of the joint MDD-E. 
\end{example}

%% file: results.tex



\begin{figure}[tb]
    \centering
    \includegraphics[width=1\linewidth]{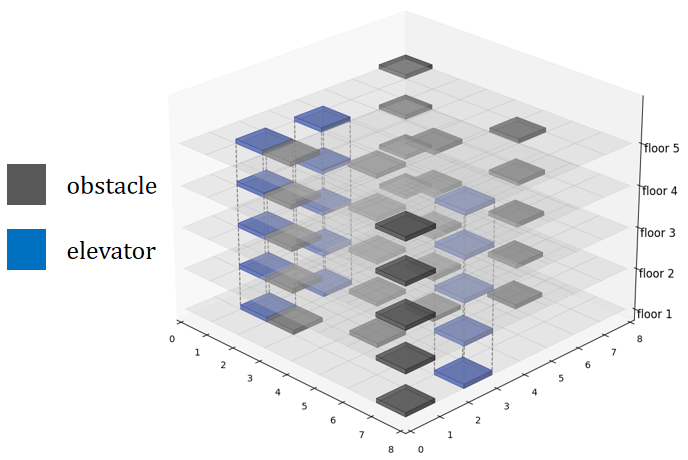}
        \vspace{-5mm}
    \caption{ A map with 5 floors connected by 3 elevators, where each floor is the \texttt{random-8-8-10} map.
      }
      \vspace{-3mm} 
     \label{fig:maps}
\end{figure}

\begin{figure*}[tb] 
    \centering
    \includegraphics[width=0.9\textwidth]{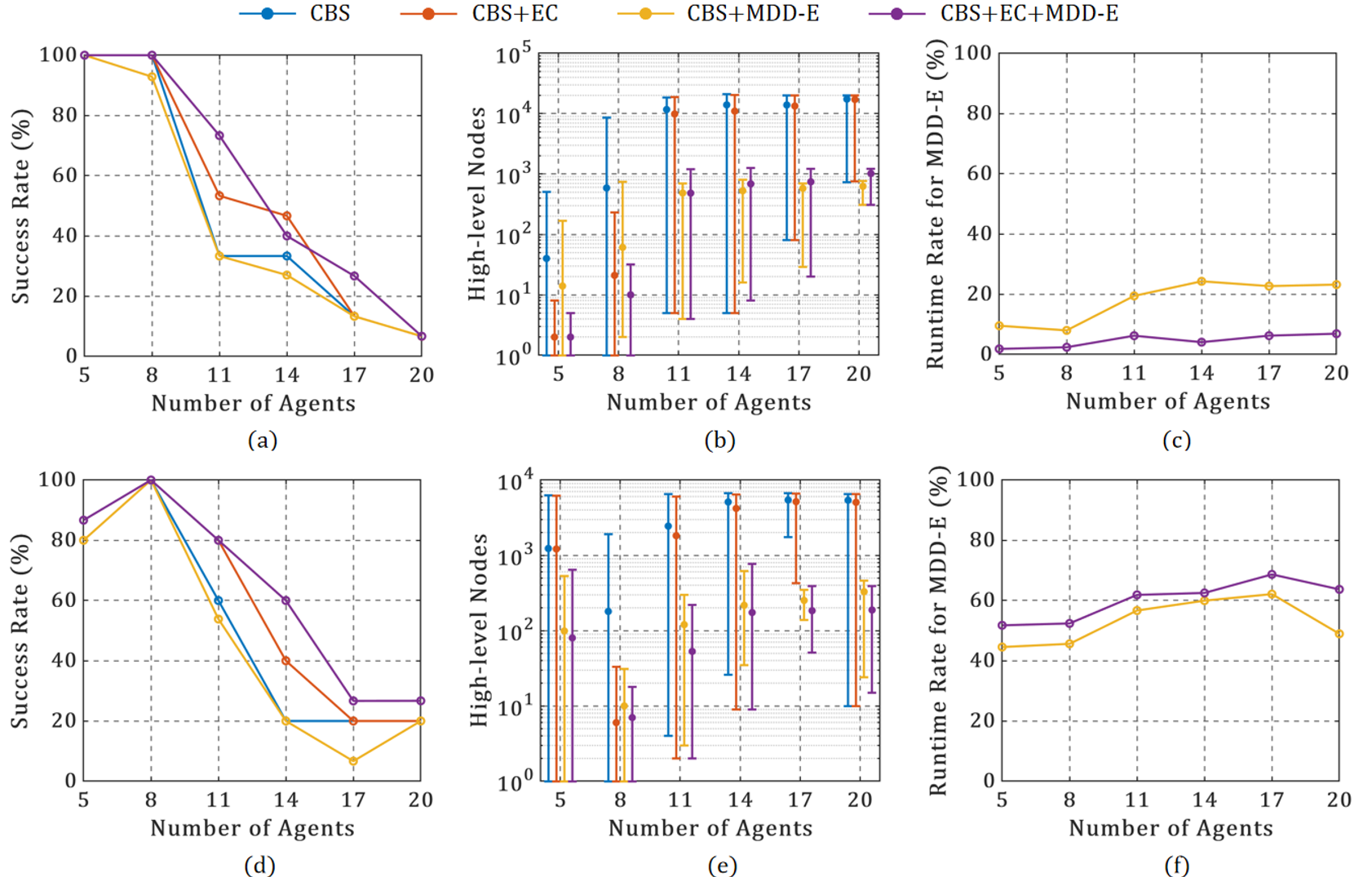} 
    \caption{(a),(d): the success rates of the algorithms when the number of agents $N \in \{5,8,11,14,17,20\}$ on \texttt{random-8-8-10} and \texttt{random-16-16-10} maps respectively.
     (b),(e): the minimum, average and maximum number of high-level nodes expanded over all instances with varying numbers of agents $N$.
    (c),(f): the fractions of total search time spent on MDD-E and related techniques in CBS+MDD-E and CBS+EC+MDD-E within the time limit over all instances with varying numbers of agents $N$.
 }
    \vspace{-2mm}
    \label{fig:resultall}
\end{figure*}
We use two 4-neighbor grids of sizes 8$\times$8 and 16$\times$16
as the single-floor maps, both with $10\%$ randomly generated obstacles, and vary the number of layers in our experiments.\footnote{
Note that, the overall scale of the workspace quickly grows as the number of layers increases.
For example, when each floor is represented by a relatively small map of size $8 \times 8$ with $64$ vertices, the resulting workspace graph contains $320$ vertices when $5$ floors are considered.}
Each floor shares the same grid map, and all the floors are connected by a certain number of elevators.
As shown in Fig.~\ref{fig:maps}, a multi-floor scenario is generated based on maps of size $8 \times 8$.
For all the algorithms, we adopt SIPP~\cite{phillips2011sipp} as the low-level planner.
All tests are on a laptop with Intel i7-13700 CPU with 32GB RAM.

\subsection{Experiment 1: Varying Types of Map}

Our first experiment considers two floors connected by 3 elevators, with \(t_{floor} = 3\).
For each map, we run 15 instances with varying numbers of agents \(N\), and set a runtime limit of 60 seconds for each instance.
In each instance, the start and goal positions of the agents, as well as the locations of the 3 elevators, are randomly generated, ensuring that a valid solution exists.
In these experiments, we compare EC and MDD-E against a naive adaptation of CBS to MAPF-E (denoted as CBS).
\subsubsection{Success Rates}
Fig.~\ref{fig:resultall}~(a), (d) show the success rates (i.e., the rate of instances solved by each algorithm within 60 seconds) for the algorithms.
Incorporating EC into the algorithms often improves the success rate within the same time limit.
The introduction of MDD-E alone brings limited improvement and may even perform worse than CBS when $N \geq 11$. This is because, as the number of agents increases, conflicts become more frequent, and the computational overhead of MDD-E in the conflict selection process restricts the number of conflicts that can be resolved within the time limit.
The computational overhead of MDD-E is often larger for agents with long paths, which is often the case since the agents take the elevator to move across multiple floors.

MDD-E demonstrates its advantages when combined with more efficient constraint-handling techniques (e.g., EC). 
For instance, on the \texttt{random-16-16-10} map, CBS+EC+MDD-E achieves a higher success rate than CBS+EC. This is because, with EC efficiently resolving conflicts, MDD-E can further guide the conflict selection process by classifying conflicts, thereby improving the success rate. 
However, on the \texttt{random-8-8-10} map, this improvement is less pronounced. In particular, when $N=14$, the success rate of CBS+EC+MDD-E is lower than that of CBS+EC in some cases. This is because, with the same number of agents, conflicts are denser on the smaller map and the additional overhead of MDD-E leads to more timeouts. Furthermore, compared with CBS+MDD-E, CBS+EC+MDD-E achieves higher success rates due to EC. In fact, on both maps, CBS+EC+MDD-E achieves a success rate that is up to around 40\% higher than CBS+MDD-E.


\subsubsection{Number of Expansions}

Fig.~\ref{fig:resultall}~(b), (e) show the average number of resolved conflicts (i.e., expanded high-level nodes).
Compared with CBS, CBS+EC resolves elevator conflicts more efficiently, leading to fewer average CT nodes and a higher success rate, as observed in both the \texttt{random-8-8-10} and \texttt{random-16-16-10} maps.
Furthermore, using MDD-E with Conflict Selection and BP techniques at the high level reduces expansions by prioritizing more critical conflicts. The CBS+EC+MDD-E algorithm can solve the same instances with fewer high-level node expansions compared to the CBS algorithm. For example, on the \texttt{random-16-16-10} map with $N \geq 11$, the number of expanded CT nodes in CBS+EC+MDD-E is much lower than that in CBS, and its success rate is also higher.

\subsubsection{\texorpdfstring{MDD-E}{MDDE} Runtime Overhead}
Fig.~\ref{fig:resultall}~(c), (f) show the fraction of total search time spent on MDD-E and related techniques within the 60-second runtime limit.
In both maps, as $N$ increases, more conflicts require repeated construction and querying of MDD-E and joint MDD-E for conflict classification, thus increasing the runtime overhead. 
 As larger maps often lead to longer paths, the MDD-E overhead increases accordingly, resulting in a higher fraction in Fig.~\ref{fig:resultall}~(f) than in Fig.~\ref{fig:resultall}~(c).

\subsection{Experiment 2: Varying Value of \texorpdfstring{$t_{floor}$}{tfloor}}

\begin{figure}[tb]
    \centering
     \vspace{3mm}
     \includegraphics[width=1\linewidth]{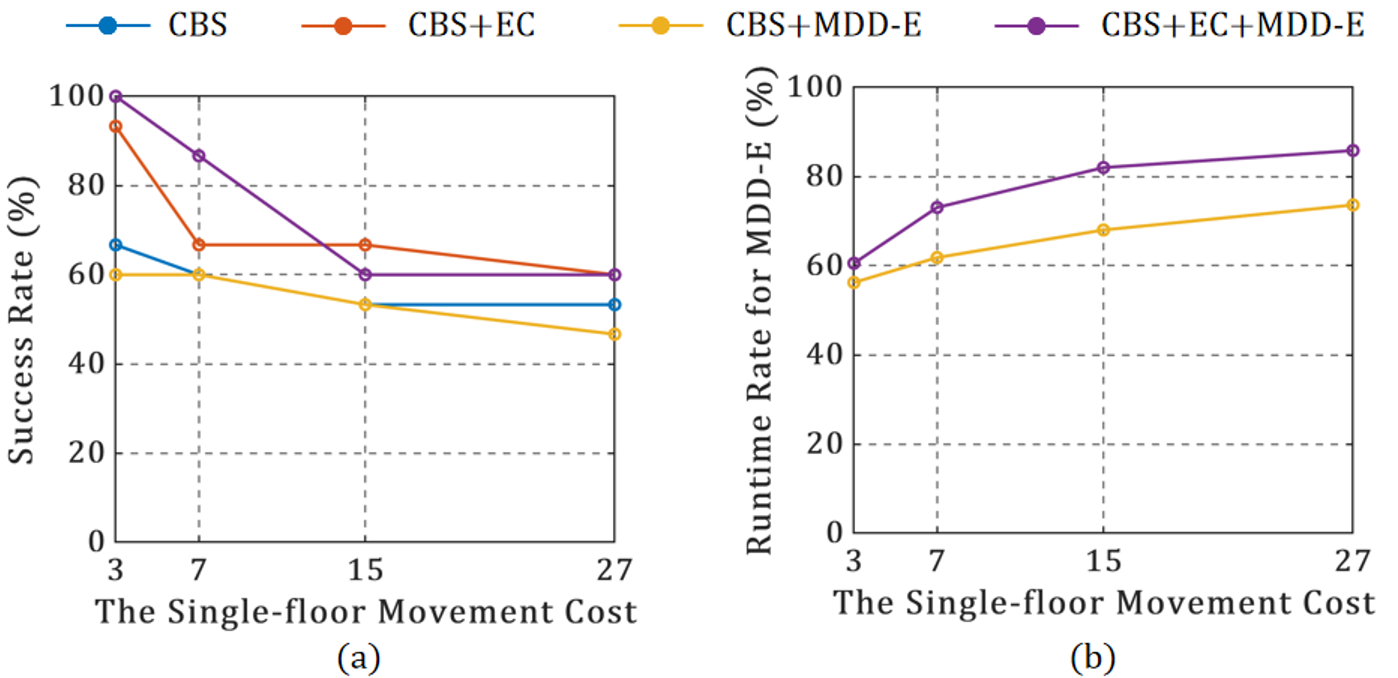}
        \vspace{-5mm}
    \caption{(a): the success rates of CBS, CBS+EC, CBS+MDD-E, and CBS+EC+MDD-E as $t_{floor}$ increases, on a 5-floor map where each floor is the \texttt{random-16-16-10} map.
     (b):  the fractions of total search time spent on MDD-E and related techniques in CBS+MDD-E and CBS+EC+MDD-E within the time limit over all instances in each group with varying values of $t_{floor}$.
\vspace{-3mm}     }
     \label{fig:changetii}
     
\end{figure}

We then investigate the effect of increasing $t_{floor}$.
We use five floors connected by three elevators, where each floor is based on the \texttt{random-16-16-10} map and the number of agents is fixed at $N=8$.
Fig.~\ref{fig:changetii}~(a) presents the success rates. 
As $t_{floor}$ increases, both CBS+EC and CBS+EC+MDD-E maintain higher success rates due to the advantages of EC compared with CBS.
Combined with Fig.~\ref{fig:changetii}~(b), which shows the fraction of total search time spent on MDD-E and related techniques, the overhead of MDD-E increases with $t_{floor}$. This, in turn, leads to a decrease in the success rates of MDD-E based algorithms, since longer agent paths result in larger and more costly MDD-E and joint MDD-E structures.  
For $t_{floor} \ge 15$, elevator conflicts in the successfully solved cases tend to be less severe and less sensitive to increases in $t_{floor}$, resulting in stable success rates.

\subsection{Experiment 3: Varying Number of Floors }

\begin{figure}[tb]
    \centering
    \includegraphics[width=1\linewidth]{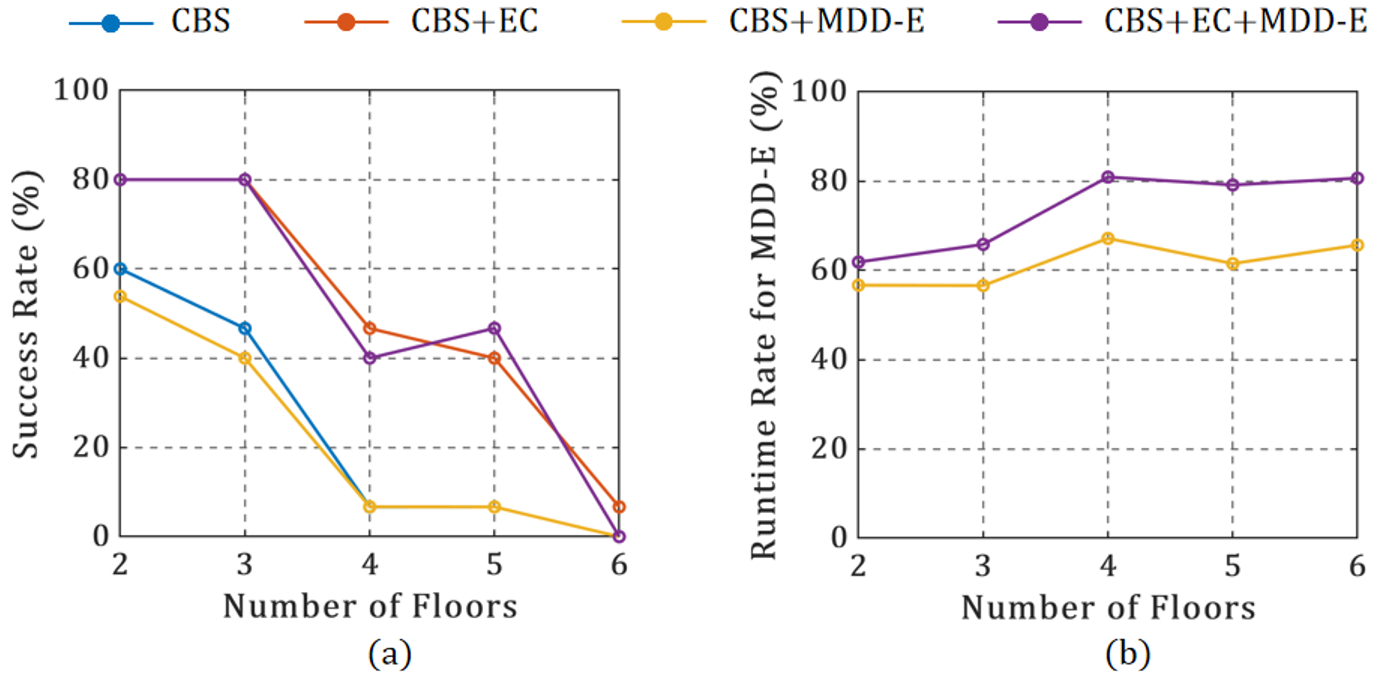}
        \vspace{-5mm}
    \caption{(a): the success rates of CBS, CBS+EC, CBS+MDD-E, and CBS+EC+MDD-E as the number of floors increases, with each floor being the \texttt{random-16-16-10} map.
     (b):  the fractions of total search time spent on MDD-E and related techniques in CBS+MDD-E and CBS+EC+MDD-E within the time limit over all instances in each group with varying numbers of floors.
\vspace{-1mm}     }
     \label{fig:changefloors}
\end{figure}
The third experiment investigates the effect of varying the number of floors with 3 elevators, where each floor is based on the \texttt{random-16-16-10} map and the number of agents is fixed at $N=11$.
Fig.~\ref{fig:changefloors}~(a) presents the success rates, showing that both CBS+EC and CBS+EC+MDD-E maintain higher success rates than CBS, and this advantage persists even as the number of floors increases due to the efficient conflict resolution provided by EC. However, as the number of floors increases, the average elevator occupied time per agent also increases, resulting in stronger coupling among agents and a gradual decline in success rates. Intuitively, the shared use of elevators introduces a combinatorial growth in potential conflicts as the number of floors increases.
Fig.~\ref{fig:changefloors}~(b) shows that the runtime ratio spent on MDD-E and related techniques grows significantly with the number of floors. When the number of floors exceeds 4, this rate in CBS+EC+MDD-E already surpasses 80\%, indicating that constructing and maintaining MDD-E imposes substantial overhead in such scenarios.



%% file: conclusion.tex
This paper investigates a new problem, MAPF-E, and develops efficient methods for resolving elevator conflicts, addressing the coupling of multiple agents introduced by elevators. By using elevator constraints, elevator conflicts can be resolved with only a single expansion while still guaranteeing optimal solutions. By extending the information stored in MDD nodes, we construct a new structure called MDD-E, which serves as the foundation for applying PC and BP techniques. Experimental results demonstrate the advantages of our new approaches under different settings. 

Future work includes fast construction of MDD-E, considering elevators holding more than one agents, extension to agents with asynchronous/continuous-time actions~\cite{ren21loosely,2025_AAAI_LSRP_ShuaiZhou,2026_AAMAS_CBSAA_XuemianWu}, and combination with task planning~\cite{ren_cbss_2023,ren21ms}.

%% file: references.bib
@inproceedings{2026_AAMAS_CBSAA_XuemianWu,
author = {Wu, Xuemian and Zhao, Shizhe and Ren, Zhongqiang},
title = {Conflict-Based Search for Multi Agent Path Finding with Asynchronous Actions},
year = {2026},
isbn = {9798400723179},
publisher = {International Foundation for Autonomous Agents and Multiagent Systems},
address = {Richland, SC},
booktitle = {Proceedings of the 25th International Conference on Autonomous Agents and Multiagent Systems},
pages = {1174–1182},
numpages = {9},
location = {Paphos, Cyprus},
series = {AAMAS '26}
}

@inproceedings{2025_AAAI_LSRP_ShuaiZhou,
  title = {Loosely Synchronized Rule-Based Planning for Multi-Agent Path Finding with Asynchronous Actions},
  author = {Zhou, Shuai and Zhao, Shizhe and Ren, Zhongqiang},
  booktitle = {Proceedings of the AAAI conference on artificial intelligence},
  year = {2025},
  month = apr,
  pages = {14763-14770},
  volume = {39},
  doi = {10.1609/aaai.v39i14.33618}
}

@inproceedings{de2013push,
  title={Push and rotate: cooperative multi-agent path planning},
  author={De Wilde, Boris and Ter Mors, Adriaan W and Witteveen, Cees},
  booktitle={Proceedings of the 2013 international conference on Autonomous agents and multi-agent systems},
  pages={87--94},
  year={2013}
}

@article{okumura2022priority,
  title = {Priority Inheritance with Backtracking for Iterative Multi-agent Path Finding},
  journal = {Artificial Intelligence},
  pages = {103752},
  year = {2022},
  issn = {0004-3702},
  doi = {https://doi.org/10.1016/j.artint.2022.103752},
  author = {Keisuke Okumura and Manao Machida and Xavier Défago and Yasumasa Tamura},
}

@inproceedings{li2021eecbs,
  title={Eecbs: A bounded-suboptimal search for multi-agent path finding},
  author={Li, Jiaoyang and Ruml, Wheeler and Koenig, Sven},
  booktitle={Proceedings of the AAAI Conference on Artificial Intelligence},
  volume={35},
  number={14},
  pages={12353--12362},
  year={2021}
}

@inproceedings{ren21loosely,
	title={Loosely Synchronized Search for Multi-agent Path Finding with Asynchronous Actions},
	author={Ren, Zhongqiang and Rathinam, Sivakumar and Choset, Howie},
	booktitle={2021 IEEE/RSJ International Conference on Intelligent Robots and Systems},
	year={2021},
	organization={IEEE}
}

@article{wagner2015subdimensional,
	title={Subdimensional expansion for multirobot path planning},
	author={Wagner, Glenn and Choset, Howie},
	journal={Artificial Intelligence},
	volume={219},
	pages={1--24},
	year={2015},
	publisher={Elsevier}
}

@article{sharon2012meta,
	title={Meta-Agent Conflict-Based Search For Optimal Multi-Agent Path Finding.},
	author={Sharon, Guni and Stern, Roni and Felner, Ariel and Sturtevant, Nathan R},
	journal={SoCS},
	volume={1},
	pages={39--40},
	year={2012}
}

@inproceedings{Wang20243DWarehouse,
  title={MAPF in 3D Warehouses: Dataset and Analysis},
  author={Qian Wang and Rishi Veerapaneni and Yu Wu and Jiaoyang Li and Maxim Likhachev},
  booktitle={Proceedings of the Thirty-Fourth International Conference on Automated Planning and Scheduling (ICAPS 2024)},
  pages={623--632},
  year={2024},
  organization={IEEE}
}

@inproceedings{guni2013ict,
	title={The increasing cost tree search for optimal multi-agent pathfinding.},
	author={Guni Sharon and Roni Stern and Meir Goldenberg and Ariel Felne},
	pages={470--495},
	year={2013},
	month = {2},
	volume = {195},
	journal = {Artificial Intelligence},
}

@inproceedings{phillips2011sipp,
	title={Sipp: Safe interval path planning for dynamic environments},
	author={Phillips, Mike and Likhachev, Maxim},
	booktitle={2011 IEEE International Conference on Robotics and Automation},
	pages={5628--5635},
	year={2011},
	organization={IEEE}
}

@inproceedings{yu2013structure_nphard,
	title={Structure and intractability of optimal multi-robot path planning on graphs},
	author={Yu, Jingjin and LaValle, Steven M},
	booktitle={Twenty-Seventh AAAI Conference on Artificial Intelligence},
	year={2013}
}

@article{sharon2015conflict,
	title={Conflict-based search for optimal multi-agent pathfinding},
	author={Sharon, Guni and Stern, Roni and Felner, Ariel and Sturtevant, Nathan R},
	journal={Artificial Intelligence},
	volume={219},
	pages={40--66},
	year={2015},
	publisher={Elsevier}
}

@inproceedings{boyarski2015icbs,
	title={ICBS: improved conflict-based search algorithm for multi-agent pathfinding},
	author={Boyarski, Eli and Felner, Ariel and Stern, Roni and Sharon, Guni and Tolpin, David and Betzalel, Oded and Shimony, Eyal},
	booktitle={Twenty-Fourth International Joint Conference on Artificial Intelligence},
	year={2015}
}

@inproceedings{barer2014suboptimal,
	title={Suboptimal variants of the conflict-based search algorithm for the multi-agent pathfinding problem},
	author={Barer, Max and Sharon, Guni and Stern, Roni and Felner, Ariel},
	booktitle={Seventh Annual Symposium on Combinatorial Search},
	year={2014}
}

@inproceedings{li2019multi,
	title={Multi-agent path finding for large agents},
	author={Li, Jiaoyang and Surynek, Pavel and Felner, Ariel and Ma, Hang and Kumar, TK Satish and Koenig, Sven},
	booktitle={Proceedings of the AAAI Conference on Artificial Intelligence},
	volume={33},
	pages={7627--7634},
	year={2019}
}

@article{ren_cbss_2023,
	title = {{CBSS}: {A} {New} {Approach} for {Multiagent} {Combinatorial} {Path} {Finding}},
	volume = {39},
	issn = {1941-0468},
	shorttitle = {{CBSS}},
	doi = {10.1109/TRO.2023.3266993},
	number = {4},
	urldate = {2024-02-27},
	journal = {IEEE Transactions on Robotics},
	author = {Ren, Zhongqiang and Rathinam, Sivakumar and Choset, Howie},
	year = {2023},
	pages = {2669--2683},
}

@article{sharon_conflict-based_2015,
  title={Conflict-based search for optimal multi-agent pathfinding},
  author={Sharon, Guni and Stern, Roni and Felner, Ariel and Sturtevant, Nathan R},
  journal={Artificial Intelligence},
  volume={219},
  pages={40--66},
  year={2015},
  publisher={Elsevier},
}

@online{zsrobotics2025,
  author    = {ZS Robotics},
  title     = {WCS system},
  year      = {2025},
  url       ={https://www.zsrobotics.com/en/zsmart.html},
  note      = {Accessed: 2025-08-16},
}

@article{li_pairwise_2021,
  title={Pairwise symmetry reasoning for multi-agent path finding search},
  author={Li, Jiaoyang and Harabor, Daniel and Stuckey, Peter J and Ma, Hang and Gange, Graeme and Koenig, Sven},
  journal={Artificial Intelligence},
  volume={301},
  pages={103574},
  year={2021},
  publisher={Elsevier}
}

@inproceedings{ren21ms,
  author = {Ren, Zhongqiang and Rathinam, Sivakumar and Choset, Howie},
  booktitle = {2021 IEEE International Conference on Robotics and Automation (ICRA)},
  title = {MS*: A New Exact Algorithm for Multi-agent Simultaneous Multi-goal Sequencing and Path Finding},
  year = {2021},
  volume = {},
  number = {},
  pages = {11560-11565},
  doi = {10.1109/ICRA48506.2021.9561779}
}
